\documentclass[10pt,a4paper,twocolumn]{article}

\usepackage{xspace}
\usepackage{amsthm}
\usepackage{amsmath}
\usepackage{balance,color}
\usepackage{pgfplots}

\usepackage{authblk} 

\newcommand{\set}[1]{\left\{#1\right\}}
\newcommand{\pr}[1]{\left(#1\right)}
\newcommand{\fpr}[1]{\mathopen{}\left(#1\right)}

\DeclareRobustCommand{\dispfunc}[2]{%
  \ensuremath{%
  \ifthenelse{\equal{#2}{}}%
    {\mathit{#1}}%
    {\mathit{#1}\fpr{#2}}}}

\newcommand{\Rcv}{R^2_{\mathit{cv}}}
\newcommand{\Rnaive}{R^2_{\mathit{naive}}}

\newtheorem{theorem}{Theorem}

\newtheorem{proposition}[theorem]{Proposition}

\pgfdeclarelayer{background}
\pgfdeclarelayer{foreground}
\pgfsetlayers{background,main,foreground}

\definecolor{yafaxiscolor}{rgb}{0.3, 0.3, 0.3}

\definecolor{yafcolor1}{rgb}{0.4, 0.165, 0.553}
\definecolor{yafcolor2}{rgb}{0.949, 0.482, 0.216}
\definecolor{yafcolor3}{rgb}{0.47, 0.549, 0.306}
\definecolor{yafcolor4}{rgb}{0.925, 0.165, 0.224}
\definecolor{yafcolor5}{rgb}{0.141, 0.345, 0.643}
\definecolor{yafcolor6}{rgb}{0.965, 0.933, 0.267}
\definecolor{yafcolor7}{rgb}{0.627, 0.118, 0.165}
\definecolor{yafcolor8}{rgb}{0.878, 0.475, 0.686}

\newlength{\yafaxispad}
\newlength{\yaftlpad}
\newlength{\yaflabelpad}
\newlength{\yafaxiswidth}
\newlength{\yafticklen}
\makeatletter
\def\pgfplots@drawtickgridlines@INSTALLCLIP@onorientedsurf#1{}
\makeatother

\newcommand{\yafdrawxaxis}[2]{
	\pgfplotstransformcoordinatex{#1}\let\xmincoord=\pgfmathresult 
	\pgfplotstransformcoordinatex{#2}\let\xmaxcoord=\pgfmathresult 
	\pgfsetlinewidth{\yafaxiswidth} 
	\pgfsetcolor{yafaxiscolor}
	\pgfpathmoveto{\pgfpointadd{\pgfpointadd{\pgfplotspointrelaxisxy{0}{0}}{\pgfqpointxy{\xmincoord}{0}}}{\pgfqpoint{-0.5\yafaxiswidth}{\yafaxispad}}}
	\pgfpathlineto{\pgfpointadd{\pgfpointadd{\pgfplotspointrelaxisxy{0}{0}}{\pgfqpointxy{\xmaxcoord}{0}}}{\pgfqpoint{0.5\yafaxiswidth}{\yafaxispad}}}
	\pgfusepath{stroke}

}
\newcommand{\yafdrawyaxis}[2]{
	\pgfplotstransformcoordinatey{#1}\let\ymincoord=\pgfmathresult 
	\pgfplotstransformcoordinatey{#2}\let\ymaxcoord=\pgfmathresult 
	\pgfsetlinewidth{\yafaxiswidth} 
	\pgfsetcolor{yafaxiscolor}
	\pgfpathmoveto{\pgfpointadd{\pgfpointadd{\pgfplotspointrelaxisxy{0}{0}}{\pgfqpointxy{0}{\ymincoord}}}{\pgfqpoint{\yafaxispad}{-0.5\yafaxiswidth}}}
	\pgfpathlineto{\pgfpointadd{\pgfpointadd{\pgfplotspointrelaxisxy{0}{0}}{\pgfqpointxy{0}{\ymaxcoord}}}{\pgfqpoint{\yafaxispad}{0.5\yafaxiswidth}}}
	\pgfusepath{stroke}
}

\newcommand{\yafdrawaxis}[4]{\yafdrawxaxis{#1}{#2}\yafdrawyaxis{#3}{#4}}

\pgfplotscreateplotcyclelist{yaf}{%
{yafcolor1,mark options={scale=0.75},mark=o}, 
{yafcolor2,mark options={scale=0.75},mark=square},
{yafcolor3,mark options={scale=0.75},mark=triangle},
{yafcolor4,mark options={scale=0.75},mark=o},
{yafcolor5,mark options={scale=0.75},mark=o},
{yafcolor6,mark options={scale=0.75},mark=o},
{yafcolor7,mark options={scale=0.75},mark=o},
{yafcolor8,mark options={scale=0.75},mark=o}} 

\pgfplotsset{axis y line=left, axis x line=bottom,
	tick align=outside,
	compat = 1.3,
	tickwidth=\yafticklen,
	clip = false,
	every axis title shift = 0pt,
    x axis line style= {-, line width = 0pt, opacity = 0},
    y axis line style= {-, line width = 0pt, opacity = 0},
    x tick style= {line width = \yafaxiswidth, color=yafaxiscolor, yshift = \yafaxispad},
    y tick style= {line width = \yafaxiswidth, color=yafaxiscolor, xshift = \yafaxispad},
    x tick label style = {font=\scriptsize, yshift = \yaftlpad},
    y tick label style = {font=\scriptsize, xshift = \yaftlpad},
    every axis y label/.style = {at = {(ticklabel cs:0.5)}, rotate=90, anchor=center, font=\scriptsize, yshift = -\yaflabelpad},
    every axis x label/.style = {at = {(ticklabel cs:0.5)}, anchor=center, font=\scriptsize, yshift = \yaflabelpad},
    x tick label style = {font=\scriptsize, yshift = 1pt},
    grid = major,
    major grid style  = {dash pattern = on 1pt off 3 pt},
	every axis plot post/.append style= {line width=\yafaxiswidth} ,
	legend cell align = left,
	legend style = {inner sep = 1pt, cells = {font=\scriptsize}},
	legend image code/.code={%
		\draw[mark repeat=2,mark phase=2,#1] 
		plot coordinates { (0cm,0cm) (0.15cm,0cm) (0.3cm,0cm) };%
	} 
}

\title{A note on adjusting $R^2$ for using with cross-validation}

\author[1,2]{Indr\.e \v{Z}liobait\.e\thanks{indre.zliobaite@helsinki.fi}}
\author[2]{Nikolaj Tatti}
\affil[1]{Dept. of Geosciences and Geography, University of Helsinki, Finland}
\affil[2]{Helsinki Institute for Information Technology HIIT, Aalto University, Finland}

\begin{document}

\maketitle

\begin{abstract}
We show how to adjust the coefficient of determination ($R^2$) when used for measuring predictive accuracy via leave-one-out cross-validation. 
\end{abstract}

\section{Background}

The coefficient of determination, denoted as  $R^2$, is commonly used in evaluating the performance of predictive models, particularly in life sciences. It indicates what proportion of variance in the target variable is explained by model predictions.  $R^2$ can be seen as a normalized version of the mean squared error. Normalization is such that $R^2=0$ is equivalent to the performance of a naive baseline always predicting a constant value, equal to the mean of the target variable. $R^2<0$ means that the performance is worse than the naive baseline. $R^2=1$ is the ideal prediction.

Given a dataset of $n$ points $R^2$ is computed as 
\begin{equation}
	R^2 = 1 - \frac{\sum_n (y_i - \hat{y}_i)^2}{\sum_n (y_i - \bar{y})^2},
\label{eq:R2}	
\end{equation}
where $\hat{y}_i$ is the prediction for $y_i$, and $\bar{y}$ is the average value of
$y_i$. Traditionally $R^2$ is computed over all data points used for model fitting. 

The naive baseline is a prediction strategy which does not use any model, but simply always predicts a constant value, equal to the mean of the target variable, that is, $\hat{y}_i = \bar{y}$. 
It follows from Eq.~(\ref{eq:R2}) that then for the naive predictor $R^2 = 0$. 

Cross-validation is a standard procedure commonly used in machine learning for
assessing out-of-sample performance of a predictive model~\cite{Hastie}. The
idea is to partition data into $k$ chunks at random, leave one chunk out from
model calibration, use that chunk for testing model performance, and continue
the same procedure with all the chunks. Leave-one-out cross-validation (LOOCV)
is used when sample size is particularly small, then the test set consists of
one data point at a time. 

When cross-validation is used, the naive baseline that always predicts a
constant value, the average value of the outputs in the training set, gives 
 $R^2<0$ if computed according to Eq.~\ref{eq:R2}. 
 This happens due to an improper normalization: the denominator in
Eq.~\ref{eq:R2} uses $\bar{y}$, and $\bar{y}$ is computed over the \emph{whole}
dataset, and not just the training data. 

\section{Cross-validated $R^2$}

To correct this, we define
\[
	\Rcv =  1 - \frac{\sum_n (y_i - \hat{y}_i)^2}{\sum_n (y_i - \bar{y}_i)^2},
\]
where $\bar{y}_i$ is the average of outputs without $y_i$,
\[
	\bar{y}_i = \frac{1}{n - 1} \sum_{j = 1, j \neq i}^n y_j\quad.
\]
That is, $\bar{y}_i$ is the naive predictor based on the training data, \emph{solely}.

We show that adjusted $\Rcv$ for leave-one-out cross-validation can be expressed as
\begin{equation}
\Rcv = \frac{R^2 - \Rnaive}{1 - \Rnaive},
\end{equation}
where $R^2$ is measured in a standard way as in Eq.~(\ref{eq:R2}), and $\Rnaive$ is the result of the naive constant predictor, and is equal to 
\begin{equation}
\Rnaive = 1 - \frac{n^2}{(n-1)^2},
\label{eq:naive}
\end{equation}
where $n$ is the number of data points. 

Figure~\ref{fig:naive} plots the standard $R^2$ score for the naive predictor, as per Eq.~(\ref{eq:naive}).

\begin{figure}
\begin{tikzpicture}
    \begin{axis}[domain=2:30,
		height=6cm,width=7cm,
		xlabel = {number of data points, $n$}, ylabel = standard $R^2$,
		xtick = {2,6,...,30},
		ymax = 0
	]
    	\addplot[yafcolor2,mark=none,very thick] {1 - (x^2)/((x-1)^2)}; 
\pgfplotsextra{\yafdrawaxis{2}{30}{-3}{0}}
\end{axis}

\end{tikzpicture}
\caption{The standard $R^2$ score for the naive constant predictor.}
\label{fig:naive}
\end{figure}
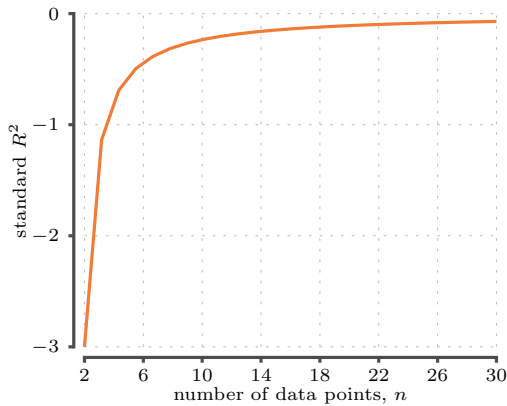

The remaining part of the paper describes mathematical proof for this adjustment. 
We will show that $\Rnaive$ does not depend on the variance of the target variable $y$, only depends on the size of the dataset $n$. 




\section{How this works}

Let us define $\Rnaive$ as the $R^2$ score for naive predictor based on training data,
\[
	\Rnaive =  1 - \frac{\sum (y_i - \bar{y}_i)^2}{\sum (y_i - \bar{y})^2}.
\]

\begin{proposition}
Let $R^2$ be the $R^2$ score of the predictor. The adjusted $R^2$ is equal to
\begin{equation}
	\Rcv = \frac{R^2 - \Rnaive}{1 - \Rnaive},
\end{equation}
where the leave-one-out cross-validated $\Rnaive$ for the constant prediction is  
\[
	\Rnaive = 1 - \pr{\frac{n}{n - 1}}^2, 
\]
where $n$ is the number of data points.
\end{proposition}


\begin{proof}
Let us write
\[
	A = \sum (y_i - \hat{y}_i)^2, \quad B = \sum (y_i - \bar{y})^2
\]
and
\[
	C = \sum (y_i - \bar{y}_i)^2.
\]
Note that $R^2 = 1 - A / B$ and $\Rcv = 1 - A / C$.

Our first step is to show that $C = \alpha B$, where $\alpha = n^2 / (n - 1)^2$.
Note that
$A$, $B$ and $C$ do not change if we translate $\set{y_i}$ by a constant;
we can assume that $n\bar{y} = \sum_{i = 1}^n y_i = 0$.

This immediately implies 
\[
	\bar{y}_i = \frac{1}{n - 1} \sum_{j = 1, j \neq i}^n y_j = \frac{-y_i}{n - 1} + n\bar{y} =  \frac{-y_i}{n - 1} .
\]

The $i$th error term of $C$ is 
\[
	(y_i - \bar{y}_i)^2 = \pr{y_i + \frac{y_i}{n - 1}}^2 =  \pr{\frac{y_in}{n - 1}}^2 = \alpha y_i^2.
\]
This leads to
\[
	C = \alpha \sum_{i = 1}^n y_i^2 = \alpha B.
\]

Finally,
\[
\begin{split}
	\frac{R^2 - \Rnaive}{1 - \Rnaive} & = \frac{R^2 - 1 +  \alpha}{\alpha} \\
	& = \frac{1 - A / B - 1 + \alpha}{\alpha} \\
	& = 1 - \frac{A}{\alpha B} = 1 - \frac{A}{C} = \Rcv,
\end{split}
\]
which concludes the proof.
\end{proof}

\bibliographystyle{plain}
\bibliography{bib_R2}

\end{document}